\PassOptionsToPackage{table}{xcolor}
\PassOptionsToPackage{dvipsnames}{xcolor}

\documentclass[mlmain]{jmlr}

\usepackage[export]{adjustbox}
\usepackage{mathtools}
\usepackage[font=small]{caption}
\usepackage{longtable}
\usepackage{bigdelim}

\usepackage{booktabs}
\usepackage[load-configurations=version-1]{siunitx} 
\usepackage{macros}

\Crefname{definition}{Definition}{Definitions}

\usepackage{amsmath}
\usepackage{nameref}

\usepackage{amssymb}

\newcommand*{\eqnameformat}[1]{%
  \textsf{#1}%
}

\makeatletter
\@ifdefinable{\org@maketag@@@}{%
  \let\org@maketag@@@\maketag@@@
  \renewcommand*{\maketag@@@}[1]{%
    \org@maketag@@@{%
      \@ifundefined{eq@name}{#1}{%
        \begin{tabular}[t]{@{}r@{}}%
          #1\tabularnewline
          \eqnameformat{\@nameuse{eq@name}}%
        \end{tabular}%
      }%
    }%
  }%
}
\newif\ifeqname@star
\newcommand*{\eqname}{%
  \@ifstar{\eqname@startrue\eqname@}{\eqname@starfalse\eqname@}%
}
\newcommand*{\eqname@}[2][]{%
  \gdef\eq@name{#1}%
  \ifx\eq@name\@empty
  \else
    \begingroup
      \@ifundefined{GetTitleString}{%
        \gdef\@currenteqlabelname{#2}%
      }{%
        \GetTitleString{#2}%
        \global\let\@currenteqlabelname\GetTitleStringResult
      }%
      \let\@currentlabelname\@currenteqlabelname
      \label{#1}%
    \endgroup
  \fi
  \gdef\eq@name{#2}%
  \ifx\eq@name\@empty
    \global\let\eq@name\relax
  \else
    \ifeqname@star
      \gdef\eq@name{\llap{#2}}%
    \fi
  \fi
}
\@ifdefinable{\org@make@display@tag}{%
  \let\org@make@display@tag\make@display@tag
  \def\make@display@tag{%
    \@ifundefined{@currenteqlabelname}{}{%
      \let\@currentlabelname\@currenteqlabelname
    }%
    \org@make@display@tag
  }%
}
\let\eq@name\relax
\let\@currenteqlabelname\relax
\g@addto@macro\displ@y@{%
  \global\let\eq@name\relax
  \global\let\@currenteqlabelname\relax
}
\@ifdefinable{\org@math@cr@@}{%
  \let\org@math@cr@@\math@cr@@
  \def\math@cr@@[#1]{%
    \org@math@cr@@[{#1}]%
    \noalign{%
      \global\let\eq@name\relax
    }%
  }%
}
\@ifdefinable{\org@eqref}{%
  \let\org@eqref\eqref
  \renewcommand*{\eqref}[1]{%
    \begingroup
      \let\eq@name\relax
      \org@eqref{#1}%
    \endgroup
  }%
}
\g@addto@macro\equation{%
  \eqname{}%
}
\makeatother

\newcommand{\cs}[1]{\texttt{\char`\\#1}}

\theorembodyfont{\upshape}
\theoremheaderfont{\scshape}
\theorempostheader{:}
\theoremsep{\newline}

\jmlrvolume{}
\firstpageno{1}
\editors{Sophia Sanborn, Christian Shewmake, Simone Azeglio, Arianna Di Bernardo, Nina Miolane}

\jmlryear{2022}
\jmlrworkshop{Symmetry and Geometry in Neural Representations}

\title[Interval Quasimetric Embeddings]{Improved Representation of Asymmetrical Distances with \\Interval Quasimetric Embeddings}

\author{\Name{Tongzhou Wang} \Email{tongzhou@mit.edu}\\
\Name{Phillip Isola} \Email{phillipi@mit.edu}\\
\addr MIT CSAIL}

\begin{document}

\maketitle

\begin{abstract}
Asymmetrical distance structures (\qmets) are ubiquitous in our lives and are gaining more attention in machine learning applications. Imposing such \qmet structures in model representations has been shown to improve many tasks, including reinforcement learning (RL) and causal relation learning. In this work, we present four desirable properties in such \qmet models, and show how prior works fail at them. We propose Interval \Qmet Embedding (IQE), which is designed to satisfy all four criteria. On three \qmet learning experiments, IQEs show strong approximation and generalization abilities, leading to better performance and improved efficiency over prior methods.
\end{abstract}
\begin{keywords}
\Qmets, Asymmetry, Representation Geometry, Representation Learning, Reinforcement Learning\par
\vspace{2.5pt}\noindent\begin{tabular*}{\textwidth}{@{}lr@{}}
\textbf{Project Page:} & \href{https://www.tongzhouwang.info/interval_quasimetric_embedding}{\small\texttt{tongzhouwang.info/interval\_quasimetric\_embedding}}\\[1.5pt]
\textbf{Quasimetric Code Package:}\hspace*{-0.8em} & \href{https://github.com/quasimetric-learning/torch-quasimetric}{\small\texttt{github.com/quasimetric-learning/torch-quasimetric}}
\end{tabular*}\vspace{-2.5pt}
\end{keywords}

\section{Introduction}
\label{sec:intro}

We live in a geometric world. As we move our arms along smooth curves in the 3-dimensional Euclidean space, or find short paths \wrt the Manhattan(-like) distance of a city, we are interacting with one of the essential geometric artifacts---distance. Distances are at the core of almost all decision making.

Although commonly modelled as symmetric quantities, distances and costs are rarely reversible. Wind, gravity, and magnetic forces naturally make one direction harder than the other. Human-designed rules, such as one-way roads, form another source of asymmetry. Fundamental quantities, time and entropy, are inherently irreversible. Decision making in this \emph{asymmetrical} world generally is based on \emph{asymmetrical distances}, called \emph{\Qmets} \citep{wang2022learning,memoli2018quasimetric}.

\Qmets capture the essence of comparing options and optimal planning---the triangle inequality, without requiring symmetry. It is a natural structure for many machine learning problems. In reinforcement learning and control, optimal goal-reaching plan costs in Markov decision processes are exactly \qmets \citep{bertsekas1991analysis,tian2020model,pitis2020inductive}. In casual inference, causal relations can also be formulated as \qmets \citep{balashankar2021learning}. Directed graph embedding and (hierarchical) relation discovery are also special cases of learning \qmets \citep{vendrov2015order,ganea2018hyperbolic,suzuki2019hyperbolic}. In such tasks, many work have demonstrated the benefit of imposing a \qmet structure in model representations \citep{wang2022learning,liu2022metric,balashankar2021learning,venkattaramanujam2019self}.

In this paper, we consider different ways to add such latent \qmet structures to model representations.  \Cref{sec:qmet-structure} discusses four desired properties: (1) satisfying \qmet constraints, (2) universal approximation, (3) low predictor parameter count, and (4) latent positive homogeneity. On a high-level, (1) and (2) ensure correct geometric inductive bias and coverage, while (3) and (4) are related to better optimization and easier usage in downstream models and/or layers. 

To our best knowledge, no prior method satisfies all four requirements. \Cref{sec:iqe} introduces Interval \Qmet Embedding (IQE) as a new \qmet embedding approach that fulfills all criteria. In \Cref{sec:expr}, we empirically verify that IQE significantly improves over previous methods on three \qmet learning tasks from \citet{wang2022learning}.

\section{Latent \Qmet Structures} \label{sec:qmet-structure}

\begin{definition}[Quasimetric]\label{defn:qmet}
Given set $\mathcal{X}$ and function $d \colon \mathcal{X}\times\mathcal{X}\rightarrow [0, \infty]$, $(\mathcal{X}, d)$ is a quasimetric space if $d$ satisfies\begin{itemize}[topsep=-3pt, itemsep=-2pt]
    \item (triangle inequality) $\forall x, y, z$, $d(x, z) \leq d(x, y) + d(y, z)$;
    \item (identity) $\forall x$, $d(x, x) = 0$.
\end{itemize} %
\end{definition}

There are many approaches to impose latent \qmet structures (\definitionref{defn:qmet}) in machine learning models. Since such models essentially capture certain \qmet distances in data, we call them \emph{\qmet models}. 

\begin{table}[t]
\floatconts
    {tab:method-comp}%
    {\caption{%
    Various \qmet modeling methods on the four criteria from \Cref{sec:qmet-structure} \textbf{(last four columns)}. Methods are grouped by their inherent latent (quasi)metric structure \textbf{(first column)}.
    {}\textsuperscript{\textdaggerdbl}PQE reparametrizes its few effective parameters (scale factors) via deep linear networks with many parameters, to overcome optimization difficulties (likely due to diminishing gradients from not being positive homogeneous)
    {}\textsuperscript{\textdagger}Deep Norm and Wide Norm were previously believed to not universally approximate \citep{wang2022learning,liu2022metric}, but we show that they in fact do in \Cref{sec:approx}.
    {}\textsuperscript{\textsection}Deep Norm and Wide Norm by default use learned concave activations to transform its components, which are not positive homogeneous (\Cref{fig:qmet-properties}).
    *MRN as described in the original paper is not a quasimetric but can be easily modified to be one. We compare with both versions in experiments. %
    \vspace{-18pt}%
    }}{%
    \resizebox{
      1\linewidth
    }{!}{%
    \renewcommand\normalsize{\small}%
    \normalsize%
    \centering%
    \newcommand{\bftab}{\fontseries{b}\selectfont}%
    \renewcommand{\arraystretch}{1.2}%
    \begin{tabular}[b]{c l c @{\hspace{14pt}} c @{\hspace{14pt}} c @{\hspace{14pt}} c }
        \toprule
           \multirow{2.2}{*}{\vspace{0pt}Latent Structure}
        &  \multicolumn{1}{c}{\multirow{2}{*}{\vspace{0pt}Method}}
        &  \multirow{2.2}{*}{\vspace{0pt}\shortstack{Quasimetric\\[1pt]Constraints}} 
        &  \multirow{2.2}{*}{\vspace{0pt}\shortstack{Universal\\[1pt]Approx.}}
        &  \multirow{2.2}{*}{\vspace{0pt}\shortstack{Latent \Qmet\\[1pt]Head \#Parameters}}
        &  \multirow{2.2}{*}{\vspace{0pt}\shortstack{Latent Positive\\[1pt]Homogeneity}}
        \\
        &
        &
        &
        &
        &
        \\
        \midrule
        \midrule

        \multirow{4.2}{*}{\shortstack{No Latent\\[1pt]Quasimetric}}
        & \multirow{2.1}{*}{\shortstack[l]{\Uncon \Qmet Predictors\\[-1pt](\eg, \citealp{tian2020model,rizi2018shortest,nair2018visual})}}
        & \cellcolor{red!25}
        & \cellcolor{Green!25}
        & 
        & 
        \\
        &
        & \multirow{-1.8}{*}{\cellcolor{red!25}\xmark}
        & \multirow{-1.8}{*}{\cellcolor{Green!25}\cmark}
        & \multirow{-1.8}{*}{---}
        & \multirow{-1.8}{*}{---}
        \\[-11pt]\\

        & \multirow{2.1}{*}{\shortstack[l]{Dot Product of Asymmetrical Encoders\\[-1pt](\eg, \citealp{schaul2015universal,hong2021bi})}}
        & \cellcolor{red!25}
        & \cellcolor{Green!25}
        & \cellcolor{Green!25}
        & \cellcolor{red!25}
        \\
        &
        & \multirow{-1.8}{*}{\cellcolor{red!25}\xmark}
        & \multirow{-1.8}{*}{\cellcolor{Green!25}\cmark}
        & \multirow{-1.8}{*}{\cellcolor{Green!25}None}
        & \multirow{-1.8}{*}{\cellcolor{red!25}\xmark}
        \\
        \midrule

        Latent Metric
        & Metric Embeddings
        & \cellcolor{Green!25}\cmark
        & \cellcolor{red!25}\xmark
        & \cellcolor{Green!25}Usually None
        & \cellcolor{Green!25}Usually \cmark
        \\
        \midrule

        \multirow{6.4}{*}{\shortstack{Latent\\[2pt]Quasimetric}}
        & \multirow{2.1}{*}{\shortstack[l]{Poisson \Qmet Embedding (PQE)\\[-1pt]\citep{wang2022learning}}}
        & \cellcolor{Green!25}
        & \cellcolor{Green!25}
        & \cellcolor{red!25}
        & \cellcolor{red!25}
        \\
        & 
        & \multirow{-1.8}{*}{\cellcolor{Green!25}\cmark}
        & \multirow{-1.8}{*}{\cellcolor{Green!25}\cmark}
        & \multirow{-1.8}{*}{\cellcolor{red!25}\shortstack{A Lot\\[-1pt](for optimization reasons)\textsuperscript{\textdaggerdbl}}}
        & \multirow{-1.8}{*}{\cellcolor{red!25}\xmark}
        \\[-11pt]\\

        & Deep Norm \citep{pitis2020inductive}
        & \cellcolor{Green!25}\cmark
        & \cellcolor{Green!25}\cmark\textsuperscript{\textdagger}
        & \cellcolor{red!25}A Lot
        & \cellcolor{red!25}\xmark\textsuperscript{\textsection}
        \\[-11pt]\\

        & Wide Norm \citep{pitis2020inductive}
        & \cellcolor{Green!25}\cmark
        & \cellcolor{Green!25}\cmark\textsuperscript{\textdagger}
        & \cellcolor{red!25}A Lot
        & \cellcolor{red!25}\xmark\textsuperscript{\textsection}
        \\[-11pt]\\

        & Metric Residual Network (MRN) \citep{liu2022metric}
        & \cellcolor{red!25}\xmark*
        & \cellcolor{Green!25}\cmark
        & \cellcolor{red!25}A Lot
        & \cellcolor{red!25}\xmark
        \\[-11pt]\\

        & \textbf{Interval Quasimetric Embedding (This paper)}
        & \cellcolor{Green!25}\cmark
        & \cellcolor{Green!25}\cmark
        & \cellcolor{Green!25}None
        & \cellcolor{Green!25}\cmark
        \\
        \bottomrule
    \end{tabular}%
    }}%
    \vspace{-4pt}%
\end{table}

\begin{figure}[t]
\centering
\floatconts
  {fig:qmet-properties}%
  {\caption{Different latent \qmets $d_\mathsf{latent}$. Plots show how predicted distances (and components forming them) change as two latent vectors move apart. Red bars show the number of trainable parameters in $d_\mathsf{latent}$. 
  (a) PQE suffers from diminishing gradients. (b,c) Deep Norm and Wide Norm require expensive latent \qmet head, and have complex relations between latents and predictions (due to its learned concave transformations). (d) IQE uses a simple head and does not suffer from gradient optimization issues. (a-d) Plots are computed at random initializations, with Deep Norm and Wide concave transformation parameters scaled to emphasize the non-linearity. %
  }}%
  {%
    \subfigure[\small Poisson \Qmet\newline Embedding \citep{wang2022learning}]{\label{fig:qmet-properties-pqe}%
      \includegraphics[scale=0.433, trim=0 13 0 6]{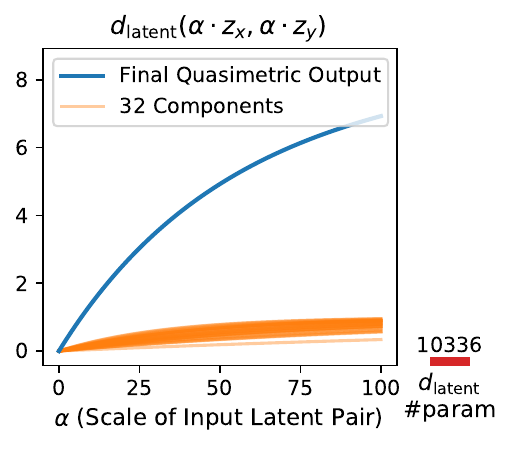}}%
    \hspace{2pt}\hfill%
    \subfigure[\small Deep Norm\newline\citep{pitis2020inductive}]{\label{fig:qmet-properties-dn}%
      \includegraphics[scale=0.433, trim=7 13 7 6]{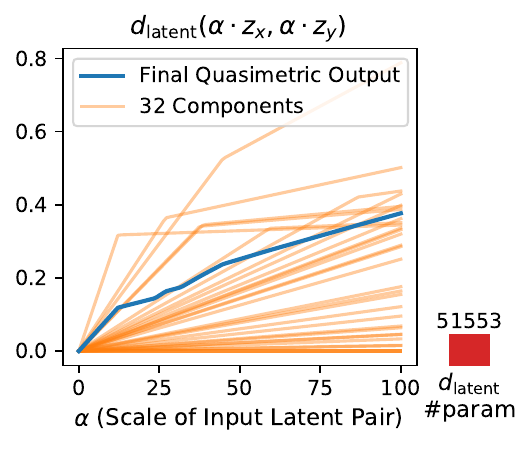}}%
    \hfill%
    \subfigure[\small Wide Norm\newline\citep{pitis2020inductive}]{\label{fig:qmet-properties-wn}%
      \includegraphics[scale=0.433, trim=7 13 7 6]{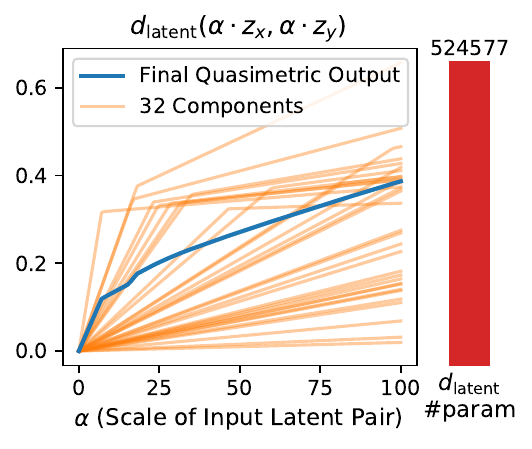}}%
    \hfill%
    \subfigure[\small Interval \Qmet\newline Embedding (Ours)]{\label{fig:qmet-properties-iqe}%
      \includegraphics[scale=0.433, trim=4 13 4 6]{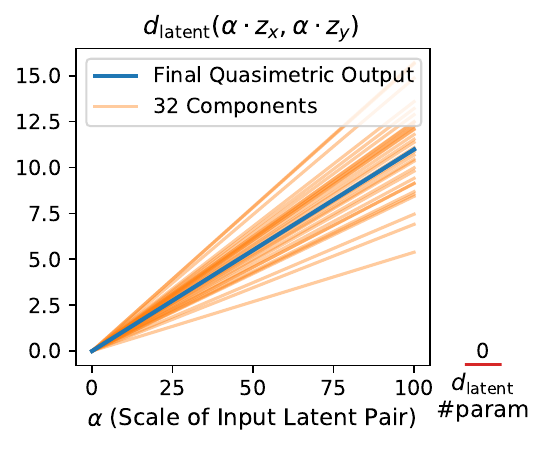}}%
    \vspace{-18pt}%
  }%
  \vspace{1.5pt}%
\end{figure}

One may simply formulate a \qmet model as an unconstrained predictor network that approximates certain \qmets \citep{schaul2015universal,nair2018visual,rizi2018shortest}.  However, the learned model may not fully respect \qmet constraints. In fact, it is proved that they may violate these constraints arbitrarily badly \citep{wang2022learning}. Alternative approaches employ a \emph{latent \qmet } function on a learned latent space \citep{pitis2020inductive,wang2022learning}. They always respect the \qmet properties by construction, but can vary in terms of expressivity, efficiency, and effectiveness for using in downstream models.

We evaluate different methods based on four criteria. First two are basic requirements for a general and proper \qmet structure,  and directly relate to the generalization capabilities (Theorem~4.3 of \citet{wang2022learning}): the formulation must be able to represent \textbf{all \qmets } and (ideally) \textbf{only \qmets}:
\begin{enumerate}[topsep=6pt, itemsep=-1pt]
    \item \textbf{(\Qmet Constraints)} \Qmet model should (at least approximately) satisfy all \qmet properties, enforcing the correct geometry and inductive biases.
    \item \textbf{(Universal Approximation)} \Qmet model should be able to generally approximate any \qmet structure of data (with jointly trained encoder models).
\end{enumerate}

To our best knowledge,  only certain \emph{latent \qmet } formulations satisfy both above properties. These methods jointly learn an \emph{encoder} $f(x; \theta)$ mapping data $x \in \mathcal{X}$ to a latent vector $z_x \in \mathcal{Z}$, as well as a (sometimes parametrized) \emph{latent \qmet head} $d_\mathsf{latent}(z_x, z_y; \psi)$ estimating the distance from data $x$ to data $y$ (via their latents). Both components collectively define a parametrized \qmet on data space $\mathcal{X}$: \begin{equation}
    \eqname{\textsf{General Latent \Qmets}}
    \hspace{105pt}d(x, y; \theta, \psi) \trieq d_\mathsf{latent}(f(x; \theta), f(y; \theta); \psi),\mathrlap{\qquad\qquad x \in \mathcal{X}, y \in \mathcal{X}.}
\end{equation}
Not all such formulation are equally good at optimization and obtaining a high-quality latent space for transferring to downstream models and/or tasks. We argue that $d_\mathsf{latent}$ should be a simple mapping that satisfies two criteria (visualized in \Cref{fig:qmet-properties}):
\begin{enumerate}[resume, topsep=6pt, itemsep=-1pt]
    \item \textbf{(Latent \Qmet Head $d_\mathsf{latent}$ with Few Parameters)} A mapping with many trainable parameters complicates the relation between input latent and \qmet distances. This not only hurts training efficiency but also makes it harder for subsequent models to take advantage of the \qmet inductive bias in latent vectors. Instead, most parameters and processing should ideally be in the encoder $f$.
    
    \item \textbf{(Latent Positive Homogeneity: $d_\mathsf{latent}(\alpha z_x, \alpha z_y) = \alpha\cdot d_\mathsf{latent}(z_x, z_y)$, $\forall \alpha > 0, z_x, z_y$)} This is similar to requiring the gradient gradient of $d_\mathsf{latent}$ (along certain directions) to have small Lipschitz constant, a property known to improve speed and stability of gradient optimization \citep{sashank2018convergence,li2019convergence}. As shown in \Cref{fig:qmet-properties-pqe}, this is not even approximately true in the prior work---Poisson \Qmet Embedding (PQE), leading to to diminishing gradient and a limited output range. Indeed, PQE requires complex reparametrization tricks to aid optimization \citep{wang2022learning}.
    This property linearly relates \qmet distances with latent magnitudes (with directions fixed), and potentially allows downstream deep models (that are good at processing Euclidean inputs) to better utilize the \qmet structure.
\end{enumerate}

\Cref{tab:method-comp} summarizes prior works and our proposal on all four properties.  Next section describes our proposed Interval Quasimetric Embedding (IQE), the only method that satisfies all four requirements.

\section{Interval Quasimetric Embeddings (IQE)}\label{sec:iqe}

The main issue with PQE is that its components are bounded in $[0, 1)$ and suffer from diminishing gradients (\Cref{fig:qmet-properties-pqe}). Special reparametrization tricks are necessary for successful optimization \citep{wang2022learning}. We propose Interval \Qmet Embeddings (IQE) to directly address this drawback. \Cref{apd:iqe-pqe} derives IQE via a modified PQE framework. 

IQE is a new encoder-based \qmet model, where a (learned) encoder maps data into some latent space, where our latent IQE \qmet $d_\mathsf{IQE}$ outputs a \qmet distance between two given latents.

\paragraph{IQE Components.}

Similar to PQE, IQE considers input latents as two-dimensional matrices (via reshaping). For input latents $u, v \in \R^{k\times l}$, IQE is formed by components that capture the total size (\ie, Lebesgue measure) of unions of several intervals on the real line:

\begin{equation}
    \eqname{\textsf{IQE Components}}
    \hspace{60pt}\mathllap{\forall i = 1, 2, \dots, k},\qquad d_i(u, v) \trieq \underbrace{\size{\bigcup_{j=1}^l \underbrace{\big[u_{ij}, \max(u_{ij}, v_{ij})\big]}_{\text{interval on the real line}}}}_{\mathclap{\text{size of the set formed from union of $l$ intervals}}}. 
\end{equation}
\Cref{fig:iqe-computation} provides a graphical illustration on how to compute these components.

\begin{figure}[t]
\centering
\floatconts
  {fig:iqe-computation}
  {\caption{Computing IQE \qmet from latent ${\color[RGB]{17,142,255}u} \in \R^{2\times 3}$ to latent ${\color[RGB]{242,114,0} v} \in \R^{2\times 3}$.}}
  {%
    \includegraphics[scale=0.2525, trim=96 535 60 100, clip]{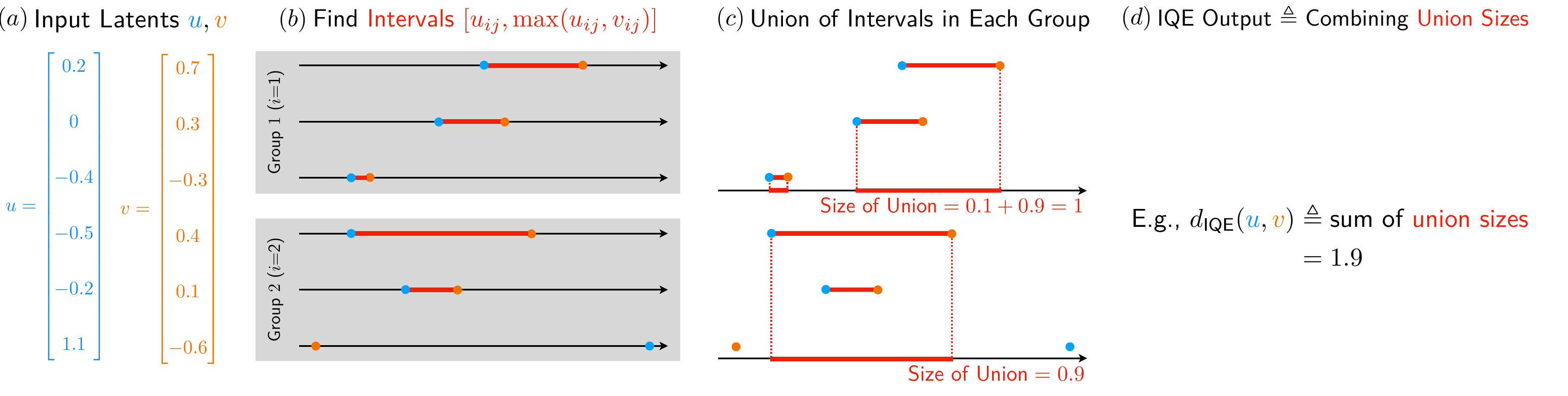}%
    \vspace{-20pt}%
  }
\end{figure}

\paragraph{Combining IQE Components.}
Unlike PQE, IQE components are positive homogeneous and can be arbitrarily scaled (\Cref{fig:qmet-properties-iqe}), and thus do not require special reparametrization in combining them. Simply summing yields the most basic yet effective IQE formulation, \text{IQE-sum}:
\begin{equation}
    \eqname{\textsf{IQE-sum}}
    d_\mathsf{IQE\hbox{-}sum}(u, v) \trieq \sum_{i=1}^k d_i(u, v) \label{eq:iqe-sum}
\end{equation}

Using the $\mathrm{maxmean}$ reduction from prior work \citep{pitis2020inductive}, we obtain \text{IQE-maxmean} with a single extra parameter $\alpha \in [0, 1]$ (parametrized via a \texttt{sigmoid} transform):
\begin{align}
    \eqname{\smash{\textsf{IQE-maxmean}}}
    d_\mathsf{IQE\hbox{-}maxmean}(u, v; \alpha) & \trieq \mathrm{maxmean}(d_1(u, v), \dots, d_k(u, v); \alpha)\label{eq:iqe-maxmean}\\
    & \trieq \alpha \cdot \max(d_1(u, v), \dots, d_k(u, v)) \notag\\
    & \hphantom{{}={}} + (1-\alpha) \cdot \mathrm{mean}(d_1(u, v), \dots, d_k(u, v))\notag
\end{align}

Prior methods often require expensive predictor heads (\eg, MRN, Deep Norm and Wide Norm) and/or complex initialization and reparametrization (\eg, PQEs). In contrast, both IQE formulations have very simple forms. In \Cref{sec:expr}, we will see that IQEs are not only simple, but also practically effective.

\subsection{Theoretical Results on Universal Approximation}\label{sec:approx}

Following prior works \citep{wang2022learning,liu2022metric,pitis2020inductive}, we assume that the target \qmet $(\mathcal{X}, d)$ has only finite distances. Here we present strong universal approximation gurantees for IQEs. All full proofs are in \Cref{apd:proofs}.

\begin{theorem}[IQE Universal Approximation; Finite Case]\label{thm:iqe-maxmean-fin}
For any finite \qmet space $(\mathcal{X}, d)$ with $\size{\mathcal{X}} = n < \infty$, there exists encoders $f_1, f_2$ such that $(f_1, d_\mathsf{IQE\hbox{-}maxmean})$ exactly represents $d$, and $(f_2, d_\mathsf{IQE\hbox{-}sum})$ approximates $d$ with distortion $\mathcal{O}(t \log^2 n)$, where $t$ is a complexity measure of $(\mathcal{X}, d)$ (called treewidth).
\end{theorem}
\begin{proof}[Sketch]
\text{IQE-maxmean} can exactly represent function $d_\mathsf{asym}(u, v) = \max_i (v_i - u_i)^+$. Rewriting $d(x, y) = \max_{z \in \mathcal{X}} (d(x, z) - d(y, z))^+$ leads a desired encoder $f_1$.\par
For \text{IQE-sum}, each IQE component can exactly represent any \qmet that takes in binary values (called \qparts) with arbitrary scaling. The desired distortion can be achieved with a convex combination of \qparts (Lemma~C.5 of \citet{wang2022learning}), and thus also with \text{IQE-sum}.
\end{proof} 

\begin{theorem}[IQE Universal Approximation; General Case]\label{thm:iqe-maxmean-inf}
Consider any \qmet space $(\mathcal{X}, d)$ where $\mathcal{X}$ is compact and $d$ is continuous. $\forall \epsilon > 0$, with sufficiently large $m$, there exists some continuous encoder $f \colon \mathcal{X} \rightarrow \R^m$ such that \begin{equation}
    \qquad\qquad \forall x \in \mathcal{X}, y \in \mathcal{X},\quad \abs{d_\mathsf{IQE\hbox{-}maxmean}(f(x), f(y)) - d(x, y)} \leq \eps.
\end{equation}
\end{theorem}

\paragraph{Relation with PQE.}  \text{IQE-maxmean} guarantees are strictly stronger than those of PQEs (and IQE-sum), which is only a distortion bound on the finite case using polynomial-sized encoders. With the same encoder, \text{IQE-maxmean} exactly represents any finite \qmet. 
\paragraph{Relation with MRN.} Our \text{IQE-maxmean} analysis is largely inspired by the MRN results. In \Cref{apd:proofs},  full proofs reduce the MRN asymmetrical component to an \text{IQE-maxmean}.
\paragraph{Deep Norm and Wide Norm.} Also using a connection to MRN, we are the first to prove that Deep Norm and Wide Norm universally approximate any \emph{\qmet}.
\begin{theorem}[Deep Norm and Wide Norm Universal Approximation]\label{thm:dnwn-ua}
Deep Norm and Wide Norm enjoy the same approximation gaurantees as stated for \text{IQE-maxmean} in \Cref{thm:iqe-maxmean-fin,thm:iqe-maxmean-inf}.
\end{theorem}

\section{Related Works}\label{sec:related}

\paragraph{\Qmet} captures a common geometric structure. Perhaps the most important example is the cost of navigating between any two states, which is extensively studied in control theory and reinforcement learning (RL) \citep{bertsekas1991analysis}. Such costs directly correspond to Q-functions and value functions \citep{schaul2015universal}, making \qmet models an increasingly popular choice for them \citep{pitis2020inductive,tian2020model,wang2022learning,liu2022metric}. In relation learning, causality \citep{balashankar2021learning} and hierarchy \citep{vendrov2015order,ganea2018hyperbolic} discovery can be modelled as learning (special cases of) \qmets, and benefit from \qmet models.

\paragraph{Latent \Qmets and Representation Learning. } Latent \qmet is the most common approach to model \qmets. A (usually jointly learned) encoder maps data into a latent space, where some (maybe parametrized) latent \qmet function gives the distance prediction output  \citep{pitis2020inductive,wang2022learning,liu2022metric}. In theoretical computer science, manually constructed \qmet embeddings are used in an improved sparse-cut algorithm \citep{memoli2018quasimetric}. In machine learning, latent \qmets can be used to obtain representation spaces that are directly informative of \qmet structures \citep{balashankar2021learning,vendrov2015order}. Indeed, representation learning methods are often designed to capture certain geometric properties, including similarity \citep{wang2020hypersphere}, and equivalences \citep{zhang2020learning,wang2022denoisedmdps}, and independence/disentanglement \citep{burgess2018understanding}.

\section{Experiments}\label{sec:expr}

IQE satisfies all four requirements from \Cref{sec:qmet-structure}. But does it perform better empirically?  

We use all three tasks from \citet{wang2022learning} to evaluate (1) the ability to approximate \qmets and generalize to test pairs (2) benefits from enforcing a \qmet structure in deep learning models. For \qmet modeling tasks, accurate prediction on test pairs (\ie, generalization) is directly related to good approximation of training distances and respecting \qmet constraints \citep{wang2022learning}. 

For all tasks, we  compare the following eight families of \qmet models: \begin{itemize}[topsep=4pt, itemsep=-1.5pt]
    \item \textbf{IQEs (Ours):} IQE-sum and IQE-maxmean.
    \item \textbf{PQEs \citep{wang2022learning}:} PQE-LH and PQE-GG.
    \item \textbf{\Uncon Neural Networks:} Unconstrained networks that map concatenated input pair to a raw distance prediction (directly, with $\exp$ transform, and with $\smash{(\cdot)^2}$ transform) or $\gamma$-discounted distance  (directly, and with a sigmoid-transform). 
    \item \textbf{Asymmetrical Dot Products:} On input pair $(x, y)$, encoding each into a feature vector with a \emph{different} network and then taking the dot product. Identical to \uncon networks, the output is used in the same $5$ ways. 
    \item \textbf{Metric Encoders:} Embedding into Euclidean space, $\ell_1$ space, hypersphere with (scaled) spherical distance, or a mixture of all three with learned weights.
    \item \textbf{Deep Norm \citep{pitis2020inductive}:} The original formulation may produce negative values. We use both the original version and a fixed version (see \Cref{sec:dn-issue}).
    \item \textbf{Wide Norm \citep{pitis2020inductive}.}
    \item \textbf{MRN \citep{liu2022metric}:} The original formulation may violate \qmet constraints. We use both the original version and a fixed version (see \Cref{sec:mrn-issue}).
\end{itemize}

\begin{table}[t]
\floatconts
    {tab:berkstan}%
    {\caption{Modeling the large-scale $\mathsf{Berkeley\hbox{-}Stanford~Web~Graph}$ with different \qmet models. For some \textbf{baseline} families, we show the \ul{best} method picked \wrt~validation set MSE.%
    \vspace{-20pt}%
    }}{%
    \centering
    \resizebox{
      1\linewidth
    }{!}{%
    \renewcommand\normalsize{\small}%
    \normalsize
    \centering
    \setlength{\tabcolsep}{5pt}
    \newcommand{\bftab}{\fontseries{b}\selectfont}%
    \renewcommand{\arraystretch}{1.4}%
    \begin{tabular}[b]{llccc}
        \toprule
        &  
        &  \multicolumn{3}{c}{Validation Set Metrics}
        \\\cmidrule{3-5}
        &  
        &  \multirow{2.2}{*}{\vspace{0pt}\shortstack{MSE \wrt $\gamma$-discounted\\[0.1ex]distances $\smash{(\times 10^{-3})}$ $\boldsymbol\downarrow$}}
        &  \multirow{2.2}{*}{\vspace{0pt}\shortstack{$\ell_1$ error when \\true $d < \infty$  $\boldsymbol\downarrow$}}
        &  \multirow{2.2}{*}{\vspace{0pt}\shortstack{Predicted distance \\ when true $d = \infty$  $\boldsymbol\uparrow$}}
        \\
        &
        &
        &
        &   \\
        \midrule
        \midrule
        
        IQE-sum
        & 
        & \bftab\round{1.0784752666950226}{3}{\color{gray}$~\pm$\round{0.05348694858770281}{3}}
        & \bftab\round{1.3025954961776733}{3}{\color{gray}$~\pm$\round{0.03140862114467444}{3}}
        & \round{118.24377746582032}{3}{\color{gray}$~\pm$\round{5.4118967785916}{3}}
        \\
        
        IQE-maxmean
        & 
        & \round{1.4875865541398525}{3}{\color{gray}$~\pm$\round{0.3065039835193903}{3}}
        & \round{1.333371901512146}{3}{\color{gray}$~\pm$\round{0.21830426684412976}{3}}
        & \round{89.63519439697265}{3}{\color{gray}$~\pm$\round{1.726092945040122}{3}}
        \\

        PQE-LH
        & 
        & \round{2.9209287371486425}{3}{\color{gray}$~\pm$\round{0.18719213143352414}{3}}
        & \round{1.6593403816223145}{3}{\color{gray}$~\pm$\round{0.047879188233684804}{3}}
        & \round{71.3904800415039}{3}{\color{gray}$~\pm$\round{0.4357155780987546}{3}}
        \\

        PQE-GG
        & 
        & \round{3.8718370720744133}{3}{\color{gray}$~\pm$\round{0.1359768019070454}{3}}
        & \round{2.121027374267578}{3}{\color{gray}$~\pm$\round{0.14593266365225357}{3}}
        & $\infty$  (overflow)
        \\

        Wide Norm
        & 
        & \round{3.532807668671012}{3}{\color{gray}$~\pm$\round{0.21201393857838322}{3}}
        & \round{1.769381594657898}{3}{\color{gray}$~\pm$\round{0.02132951306760224}{3}}
        & \round{124.65802764892578}{3}{\color{gray}$~\pm$\round{2.8677826229958154}{3}}
        \\[3pt]

        \multirow{2}{*}{Deep Norm}
        & (Original)
        & \round{5.071479082107544}{3}{\color{gray}$~\pm$\round{0.13476357031416314}{3}}
        & \round{2.0852701663970947}{3}{\color{gray}$~\pm$\round{0.0632648640792226}{3}}
        & \round{120.04515838623047}{3}{\color{gray}$~\pm$\round{4.352518794402874}{3}}
        \\[-1pt]

        & (+ Non-Negativity Fix)
        & \round{4.7599999234080315}{3}{\color{gray}$~\pm$\round{0.35420881108181046}{3}}
        & \round{2.0351162195205688}{3}{\color{gray}$~\pm$\round{0.056755020242355785}{3}}
        & \round{120.1505630493164}{3}{\color{gray}$~\pm$\round{4.69954206104865}{3}}
        \\[3pt]

        \multirow{2}{*}{MRN}
        & (Original)
        & \round{10.820086672902107}{3}{\color{gray}$~\pm$\round{0.8170756358954667}{3}}
        & \round{2.882313108444214}{3}{\color{gray}$~\pm$\round{0.20536220594789711}{3}}
        & \round{129.52826538085938}{3}{\color{gray}$~\pm$\round{4.237154491193482}{3}}
        \\[-1pt]

        & (+ \Qmet Fix)
        & \round{6.8754578940570354}{3}{\color{gray}$~\pm$\round{0.3326335049546617}{3}}
        & \round{2.508049249649048}{3}{\color{gray}$~\pm$\round{0.0907796124217027}{3}}
        & \round{129.914208984375}{3}{\color{gray}$~\pm$\round{6.2905016545869925}{3}}
        \\
        \midrule

        \ul{Best} Metric Embedding
        & 
        & \round{17.595218867063522}{3}{\color{gray}$~\pm$\round{0.26668822519713126}{3}}
        & \round{7.539879989624024}{3}{\color{gray}$~\pm$\round{0.07420489860935446}{3}}
        & \round{53.85003433227539}{3}{\color{gray}$~\pm$\round{3.843024787898114}{3}}
        \\

        \midrule

        \multirow{2}{*}{\vspace{0pt}\ul{Best} \Uncon Net.}
        & (No Regularizer)
        & \round{3.0862420331686735}{3}{\color{gray}$~\pm$\round{0.03916676919711341}{3}}
        & \round{2.1151447772979735}{3}{\color{gray}$~\pm$\round{0.024078517111510263}{3}}
        & \round{59.5242919921875}{3}{\color{gray}$~\pm$\round{0.37004197274349077}{3}}
        \\[-1pt]
        
        & (+ $\Delta$-Ineq.~Regularizer)
        & \round{2.8128044679760933}{3}{\color{gray}$~\pm$\round{0.06251623447057535}{3}}
        & \round{2.210947322845459}{3}{\color{gray}$~\pm$\round{0.034070232889604174}{3}}
        & \round{61.37094268798828}{3}{\color{gray}$~\pm$\round{0.3935568080866707}{3}}
        \\[3pt]

        \multirow{2}{*}{\ul{Best} Asym.~Dot Product}
        & (No Regularizer)
        & \round{48.105619102716446}{3}{\color{gray}$~\pm$\round{0.005621992204445349}{3}}
        & 2.520 $\times 10^{11}${\color{gray}~$\pm$2.175 $\times 10^{11}$}
        & 2.679 $\times 10^{11}${\color{gray}~$\pm$2.540 $\times 10^{11}$}
        \\[-1pt]
        
        & (+ $\Delta$-Ineq.~Regularizer)
        & \round{48.102133721113205}{3}{\color{gray}$~\pm$\round{0.0002099325861824455}{3}}
        & 2.299 $\times 10^{11}${\color{gray}~$\pm$9.197 $\times 10^{10}$}
        & 2.500 $\times 10^{11}${\color{gray}~$\pm$1.446 $\times 10^{11}$}
        \\
        \bottomrule
    \end{tabular}%
    }}
    \vspace{-2pt}
\end{table}

\paragraph{\Uncon Models and A Triangle Inequality Regularizer.} Unlike other methods that explicitly enforce \qmet structures, both \uncon networks and and asymmetrical dot products can represent any function, and are \uncon models.  We evaluate these methods because (1) that they are widely used  \citep{tian2020model,hong2021bi,rizi2018shortest,schaul2015universal} and (2) that their performances reveal whether standard training of generic models can somehow implicitly learn the underlying \qmet structure in data.  Additionally, we test whether explicit regularization can help these \uncon models better learn \qmet structure, and train them with a triangle inequality regularizer ${\mathbb{E}_{x,y,z}\big[\max(0, \gamma^{\hat{d}(x, y) + \hat{d}(y, z)} - \gamma^{\hat{d}(x, z)})^2\big]}$ for weights $\in \{0.3, 1, 3\}$.

All results are aggregated from $5$ seeds. Full details are provided in \Cref{sec:expr-details}.

\subsection{Large-Scale Social Graph}\label{sec:berkstan}

$\mathsf{Berkeley\hbox{-}Stanford~Web~Graph}$ \citep{snapnets} is a large real-wold social graph, containing $685{,}230$ pages as nodes, and $7{,}600{,}595$ hyperlinks as directed edges. Following prior work \citep{wang2022learning}, we use $128$-dimensional $\mathsf{node2vec}$ features \citep{grover2016node2vec} and the landmark method \citep{rizi2018shortest} to obtain a training set of $2{,}500{,}000$ pairs, and a validation set of $150{,}000$ pairs. 

\paragraph{IQEs significantly improve modeling large real-world graphs.} We train various \qmet models to approximate the training distances by minimizng MSE \wrt $\gamma$-discounted distance with $\gamma = 0.9$. In \Cref{tab:berkstan}, both IQEs greatly outperform all baselines, attaining lowest MSE, accurately predicting finite distances, and outputting high predictions for infinite (unreachable) pairs. Compared to the prior best methods, the simple IQE-sum has a $61\%$ improvement on MSE and a $16\%$ improvement on $\ell_1$ error (on finite distances).

\paragraph{Effects of $k$ and $l$.} Both IQE and PQE interpret input latent vectors as two-dimensional $\in \R^{k\times l}$, and computes $k$ components each from a $l$-dimensional subspace. With the fixed total latent dimension as $512$, we vary $k$ and $l$ choices and plot their effects in \Cref{fig:berkstan-k-l-ablation}. \begin{itemize}[topsep=3pt, itemsep=-2pt]
\item \textbf{IQEs} generally approximate and generalize better (lower test MSE) with larger component size $l$. More components (large $k$) often lead to larger predictions for unreachable pairs in IQE-sum, at the cost of worse approximation with smaller $l$. Small $k$ and large $l$ usually perform well, but the best results are from using a large component size $l$ while still maintaining some number of components (\eg, $k \geq 8$ for IQE-sum).
\item \textbf{PQE} behavior is largely unaffected by this choice, and generally underperforms IQE except for a few extreme choices. With large $k$, PQE-GG tends to overflow when predicting on unreachable pairs, while PQE-LH does not suffer from this issue.
\end{itemize}

\paragraph{Strict \qmet structure is better than regularizing \uncon models.} As shown in \Cref{tab:berkstan}, adding a triangle inequality regularizer only has marginal benefits, and is still significantly worse than the strictly enforced \qmet structure from IQEs.

\begin{figure}[t]
\centering
\floatconts
  {fig:berkstan-k-l-ablation}
  {\caption{Effect of different $(k,l)$ choices for IQEs and PQEs with fixed total latent dimension $=512$.}}
  {%
    \includegraphics[scale=0.5175, trim=10 15 0 11]{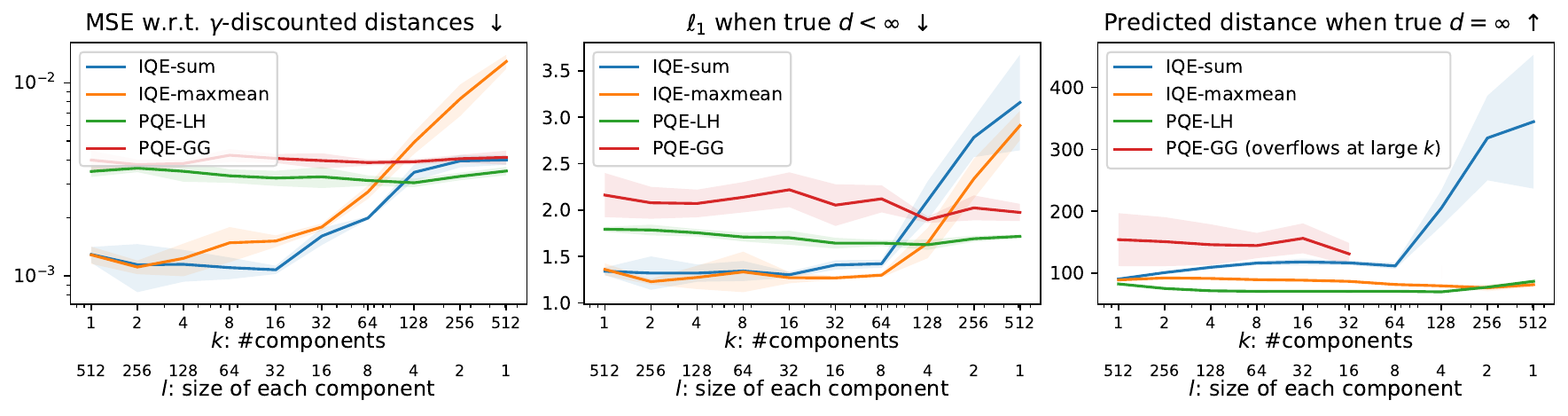}%
    \vspace{-16pt}%
  }
\end{figure}

\subsection{Random Graphs}\label{sec:rand-graph}

Using three randomly generated graphs, we compare different \qmet models' ability to fit different \qmet structures with different training set sizes (by regressing graph distances in a similar fashion as above), and to generalize on test pairs. We visualize the distinct structures of three graphs \Cref{fig:graph-data-pdist} in the appendix.

\paragraph{IQEs are simple, efficient and effective. } In \Cref{fig:random-graph}, IQEs are consistently among the top performing methods for all three distinct structures, when many other latent \qmet baselines use $12{,}500$ more parameters. While Deep Norm slightly outperforms IQEs on the sparse graph with a large training set size, IQEs are consistently better with smaller training sets on all graphs, and comparable in other settings.

\begin{figure}[t]
\centering
\floatconts
  {fig:random-graph}
  {\caption{Modeling graphs of different structures. Deep Norm, Wide Norm and MRN use latent \qmet head with $12{,}500$ more parameters than IQEs ($\leq 1$ parameter) and PQEs. The much simpler IQEs are comparable or better than them, and outperform all other methods.%
  \vspace{4pt}%
  }}
  {%
    \includegraphics[scale=0.459, trim=6 8 0 12]{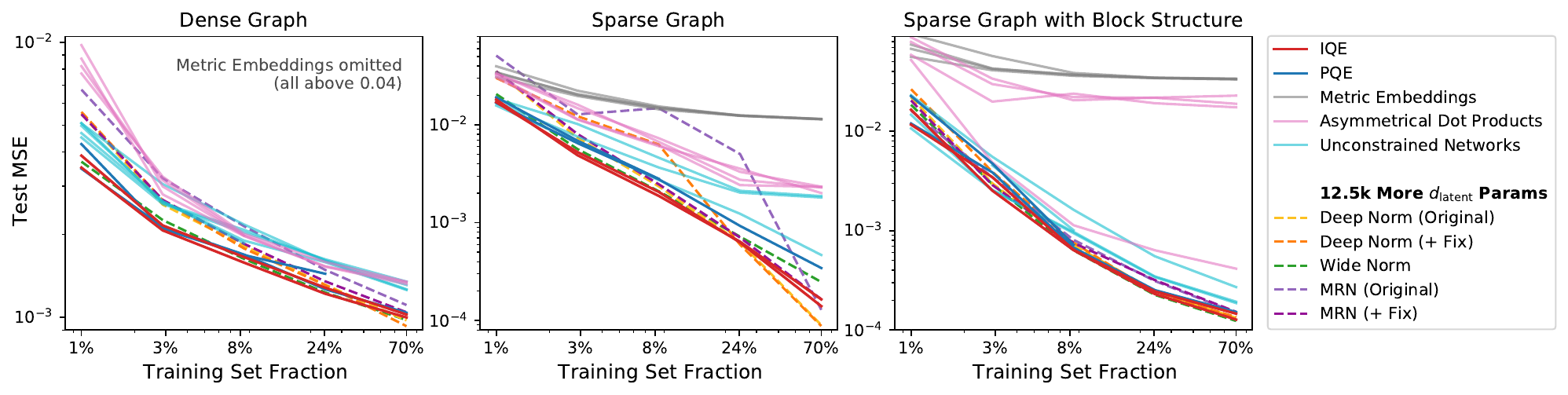}%
    \vspace{-19pt}%
  }
\end{figure}

\begin{figure}[t]
\centering
\floatconts
  {fig:q-learning}%
  {\caption{Offline goal-conditioned Q-learning results on a simple grid-world with four directional actions. Using different goal-conditioned Q-function models leads to different inductive biases and planning success rates. We use one-step greedy planning \wrt learned Q-function.}}%
  {%
    \subfigure[\small Grid-world with\newline one-way doors.]{\label{fig:gridworld}%
      \includegraphics[scale=0.254, trim=3 -36 0 0]{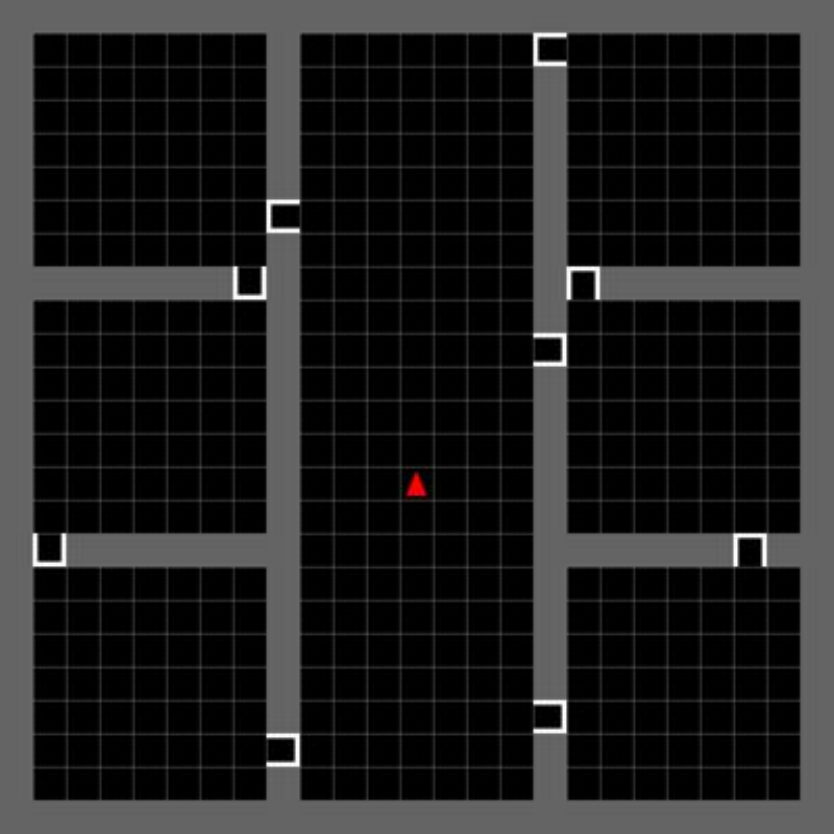}%
    }%
    \hfill%
    \subfigure[\small Planning with different Q-function models.]{\label{fig:q-learning-psucc}%
      \includegraphics[scale=0.313, trim=-5 18 1 8]{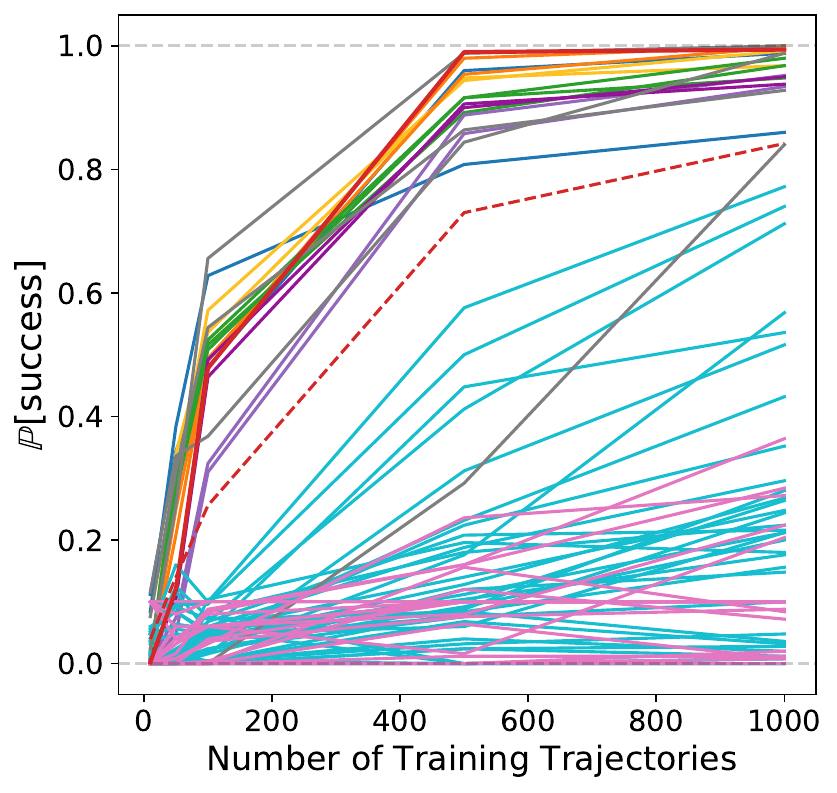}%
    }%
    \hfill\hspace{3pt}%
    \subfigure[\small Comparing top-performing methods with sufficiently many training trajectories.]{\label{fig:q-learning-psucc-zoom}%
      \includegraphics[scale=0.313, trim=2 18 0 8]{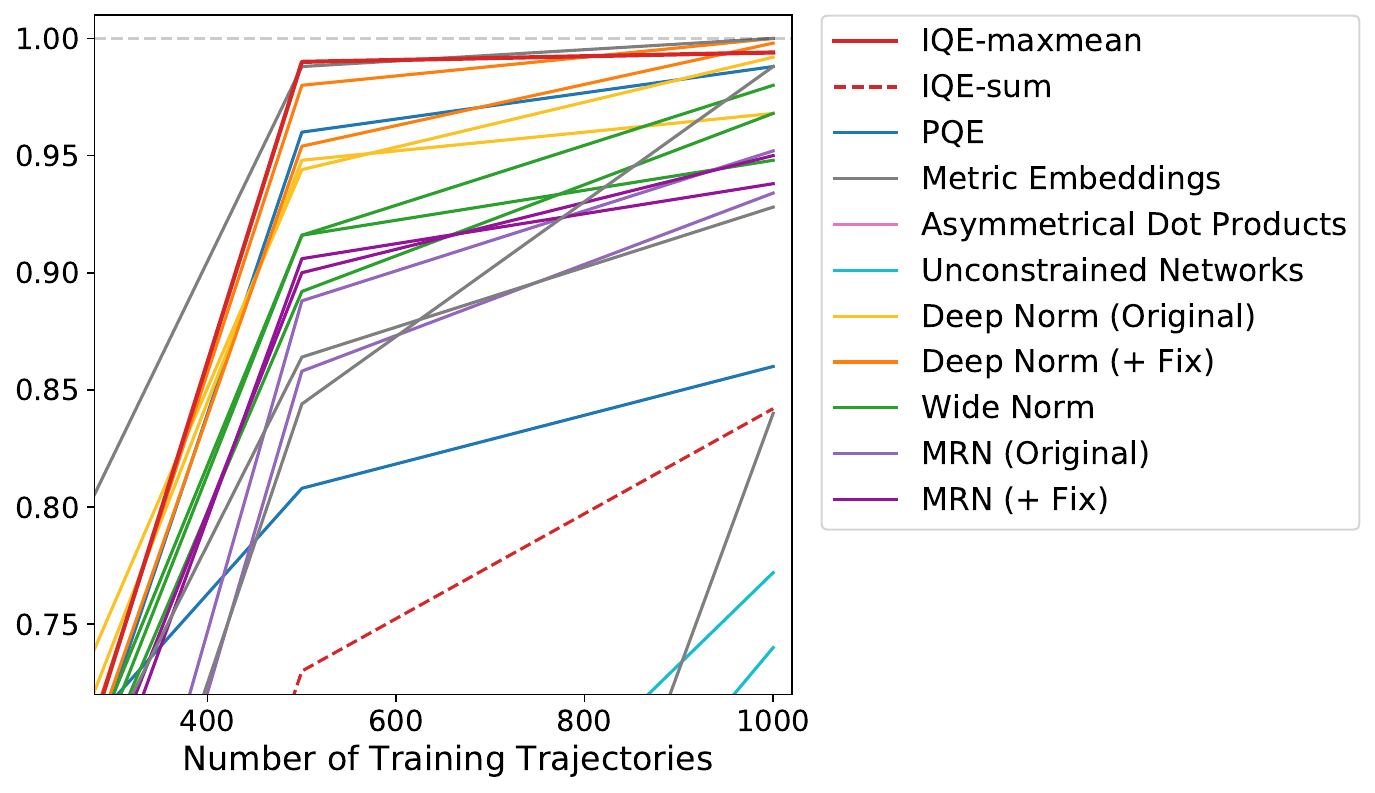}\hspace{-6.5pt}%
    }%
    \vspace{-18pt}%
  }
\end{figure}

\subsection{Offline Q-Learning}
In planning, optimal plan costs to reach target states have a \qmet structure \citep{bertsekas1991analysis,tian2020model,wang2022learning}. We use \qmet models as drop-in replacements for goal-conditioned Q-function models in offline Q-learning on grid-world environment with one-way doors (\Cref{fig:gridworld}; \citealp{wang2022learning}). 

\paragraph{\Qmet structure improves sample efficiency of offline RL.} In \Cref{fig:q-learning-psucc}, \qmet models generally outperform the widely used unconstrained networks and asymmetrical dot products, which do not have similar geometric constraints. Deep Norm and MRN also benefit from our proposed fixes that strictly enforce \qmet constraints. Indeed, \qmet structures greatly improve sample efficiency, and leads to better planning results with fewer training trajectories. As training set size increases, IQE-maxmean outperforms all other latent \qmet baselines, and only Deep Norm (with our  fix) remains comparable but requires many more parameters  (\Cref{fig:q-learning-psucc-zoom}).

\paragraph{IQE-maxmean is effective in offline RL. } IQE-maxmean performs significantly better than  IQE-sum on this task. We suspect that the $\max$ operation in the $\mathrm{maxmean}$ reduction may encourage the learned function to be more conservative (\ie, outputting larger distances), which improves offline RL \citep{kumar2020conservative}. Indeed, other methods that also use the $\mathrm{maxmean}$ reduction (Deep Norm and Wide Norm) generally perform better than MRN and PQEs, which use a summation reduction. IQE-maxmean outperforms most other such methods that use $\mathrm{maxmean}$ reduction, suggesting that its effectiveness is not just due to the reduction, but also the IQE formulation.

\section{Discussion}

In this work, we present four desired criteria when adding \qmet structures to machine learning models (\Cref{sec:qmet-structure}). Our proposed Interval \Qmet Embedding (IQE) is the first method that satisfies all four (\Cref{sec:iqe}), and has strong theoretical guarantees (\Cref{sec:approx}) and empirical performance (\Cref{sec:expr}). We believe that IQE's simple yet powerful form can enable more machine learning applications of \qmets in modeling asymmetrical geometric structures, and that our four criteria are helpful in developing novel and better \qmet structures.

\section*{Acknowledgements}
TW was supported in part by ONR MURI grant N00014-22-1-2740.

\bibliography{references}

\begin{thebibliography}{28}
\providecommand{\natexlab}[1]{#1}
\providecommand{\url}[1]{\texttt{#1}}
\expandafter\ifx\csname urlstyle\endcsname\relax
  \providecommand{\doi}[1]{doi: #1}\else
  \providecommand{\doi}{doi: \begingroup \urlstyle{rm}\Url}\fi

\bibitem[Balashankar and Subramanian(2021)]{balashankar2021learning}
Ananth Balashankar and Lakshminarayanan Subramanian.
\newblock Learning faithful representations of causal graphs.
\newblock In \emph{Proceedings of the 59th Annual Meeting of the Association
  for Computational Linguistics and the 11th International Joint Conference on
  Natural Language Processing (Volume 1: Long Papers)}, pages 839--850, 2021.

\bibitem[Bertsekas and Tsitsiklis(1991)]{bertsekas1991analysis}
Dimitri~P Bertsekas and John~N Tsitsiklis.
\newblock An analysis of stochastic shortest path problems.
\newblock \emph{Mathematics of Operations Research}, 16\penalty0 (3):\penalty0
  580--595, 1991.

\bibitem[Burgess et~al.(2018)Burgess, Higgins, Pal, Matthey, Watters,
  Desjardins, and Lerchner]{burgess2018understanding}
Christopher~P Burgess, Irina Higgins, Arka Pal, Loic Matthey, Nick Watters,
  Guillaume Desjardins, and Alexander Lerchner.
\newblock Understanding disentangling in $\beta$-{VAE}.
\newblock \emph{arXiv preprint arXiv:1804.03599}, 2018.

\bibitem[Chevalier-Boisvert et~al.(2018)Chevalier-Boisvert, Willems, and
  Pal]{gym_minigrid}
Maxime Chevalier-Boisvert, Lucas Willems, and Suman Pal.
\newblock Minimalistic gridworld environment for openai gym.
\newblock \url{https://github.com/maximecb/gym-minigrid}, 2018.

\bibitem[Ganea et~al.(2018)Ganea, B{\'e}cigneul, and
  Hofmann]{ganea2018hyperbolic}
Octavian Ganea, Gary B{\'e}cigneul, and Thomas Hofmann.
\newblock Hyperbolic entailment cones for learning hierarchical embeddings.
\newblock In \emph{International Conference on Machine Learning}, pages
  1646--1655. PMLR, 2018.

\bibitem[Grover and Leskovec(2016)]{grover2016node2vec}
Aditya Grover and Jure Leskovec.
\newblock node2vec: Scalable feature learning for networks.
\newblock In \emph{Proceedings of the 22nd ACM SIGKDD international conference
  on Knowledge discovery and data mining}, pages 855--864, 2016.

\bibitem[Hong et~al.(2021)Hong, Yang, and Agrawal]{hong2021bi}
Zhang-Wei Hong, Ge~Yang, and Pulkit Agrawal.
\newblock Bi-linear value networks for multi-goal reinforcement learning.
\newblock In \emph{International Conference on Learning Representations}, 2021.

\bibitem[Ioffe and Szegedy(2015)]{ioffe2015batch}
Sergey Ioffe and Christian Szegedy.
\newblock Batch normalization: Accelerating deep network training by reducing
  internal covariate shift.
\newblock In \emph{International Conference on Machine Learning}, pages
  448--456, 2015.

\bibitem[Kingma and Ba(2014)]{kingma2014adam}
Diederik~P Kingma and Jimmy Ba.
\newblock Adam: A method for stochastic optimization.
\newblock \emph{arXiv preprint arXiv:1412.6980}, 2014.

\bibitem[Kumar et~al.(2020)Kumar, Zhou, Tucker, and
  Levine]{kumar2020conservative}
Aviral Kumar, Aurick Zhou, George Tucker, and Sergey Levine.
\newblock Conservative q-learning for offline reinforcement learning.
\newblock \emph{Advances in Neural Information Processing Systems},
  33:\penalty0 1179--1191, 2020.

\bibitem[Leskovec and Krevl(2014)]{snapnets}
Jure Leskovec and Andrej Krevl.
\newblock {SNAP Datasets}: {Stanford} large network dataset collection.
\newblock \url{http://snap.stanford.edu/data}, June 2014.

\bibitem[Li and Orabona(2019)]{li2019convergence}
Xiaoyu Li and Francesco Orabona.
\newblock On the convergence of stochastic gradient descent with adaptive
  stepsizes.
\newblock In \emph{The 22nd international conference on artificial intelligence
  and statistics}, pages 983--992. PMLR, 2019.

\bibitem[Liu et~al.(2022)Liu, Feng, Liu, and Stone]{liu2022metric}
Bo~Liu, Yihao Feng, Qiang Liu, and Peter Stone.
\newblock Metric residual networks for sample efficient goal-conditioned
  reinforcement learning.
\newblock \emph{arXiv preprint arXiv:2208.08133}, 2022.

\bibitem[Loshchilov and Hutter(2016)]{loshchilov2016sgdr}
Ilya Loshchilov and Frank Hutter.
\newblock {SGDR}: Stochastic gradient descent with warm restarts.
\newblock \emph{arXiv preprint arXiv:1608.03983}, 2016.

\bibitem[M{\'e}moli et~al.(2018)M{\'e}moli, Sidiropoulos, and
  Sridhar]{memoli2018quasimetric}
Facundo M{\'e}moli, Anastasios Sidiropoulos, and Vijay Sridhar.
\newblock Quasimetric embeddings and their applications.
\newblock \emph{Algorithmica}, 80\penalty0 (12):\penalty0 3803--3824, 2018.

\bibitem[Nair et~al.(2018)Nair, Pong, Dalal, Bahl, Lin, and
  Levine]{nair2018visual}
Ashvin~V Nair, Vitchyr Pong, Murtaza Dalal, Shikhar Bahl, Steven Lin, and
  Sergey Levine.
\newblock Visual reinforcement learning with imagined goals.
\newblock \emph{Advances in neural information processing systems}, 31, 2018.

\bibitem[Pitis et~al.(2020)Pitis, Chan, Jamali, and Ba]{pitis2020inductive}
Silviu Pitis, Harris Chan, Kiarash Jamali, and Jimmy Ba.
\newblock An inductive bias for distances: Neural nets that respect the
  triangle inequality.
\newblock In \emph{Proceedings of the Eighth International Conference on
  Learning Representations}, 2020.

\bibitem[Rizi et~al.(2018)Rizi, Schloetterer, and Granitzer]{rizi2018shortest}
Fatemeh~Salehi Rizi, Joerg Schloetterer, and Michael Granitzer.
\newblock Shortest path distance approximation using deep learning techniques.
\newblock In \emph{2018 IEEE/ACM International Conference on Advances in Social
  Networks Analysis and Mining (ASONAM)}, pages 1007--1014. IEEE, 2018.

\bibitem[Sashank et~al.(2018)Sashank, Satyen, and
  Sanjiv]{sashank2018convergence}
J~Reddi Sashank, Kale Satyen, and Kumar Sanjiv.
\newblock On the convergence of adam and beyond.
\newblock In \emph{International Conference on Learning Representations},
  volume~5, page~7, 2018.

\bibitem[Schaul et~al.(2015)Schaul, Horgan, Gregor, and
  Silver]{schaul2015universal}
Tom Schaul, Daniel Horgan, Karol Gregor, and David Silver.
\newblock Universal value function approximators.
\newblock In \emph{International conference on machine learning}, pages
  1312--1320. PMLR, 2015.

\bibitem[Suzuki et~al.(2019)Suzuki, Takahama, and Onoda]{suzuki2019hyperbolic}
Ryota Suzuki, Ryusuke Takahama, and Shun Onoda.
\newblock Hyperbolic disk embeddings for directed acyclic graphs.
\newblock In \emph{International Conference on Machine Learning}, pages
  6066--6075. PMLR, 2019.

\bibitem[Tian et~al.(2020)Tian, Nair, Ebert, Dasari, Eysenbach, Finn, and
  Levine]{tian2020model}
Stephen Tian, Suraj Nair, Frederik Ebert, Sudeep Dasari, Benjamin Eysenbach,
  Chelsea Finn, and Sergey Levine.
\newblock Model-based visual planning with self-supervised functional
  distances.
\newblock \emph{arXiv preprint arXiv:2012.15373}, 2020.

\bibitem[Vendrov et~al.(2015)Vendrov, Kiros, Fidler, and
  Urtasun]{vendrov2015order}
Ivan Vendrov, Ryan Kiros, Sanja Fidler, and Raquel Urtasun.
\newblock Order-embeddings of images and language.
\newblock \emph{arXiv preprint arXiv:1511.06361}, 2015.

\bibitem[Venkattaramanujam et~al.(2019)Venkattaramanujam, Crawford, Doan, and
  Precup]{venkattaramanujam2019self}
Srinivas Venkattaramanujam, Eric Crawford, Thang Doan, and Doina Precup.
\newblock Self-supervised learning of distance functions for goal-conditioned
  reinforcement learning.
\newblock \emph{arXiv preprint arXiv:1907.02998}, 2019.

\bibitem[Wang and Isola(2020)]{wang2020hypersphere}
Tongzhou Wang and Phillip Isola.
\newblock Understanding contrastive representation learning through alignment
  and uniformity on the hypersphere.
\newblock In \emph{International Conference on Machine Learning}, pages
  9929--9939. PMLR, 2020.

\bibitem[Wang and Isola(2022)]{wang2022learning}
Tongzhou Wang and Phillip Isola.
\newblock On the learning and learnability of quasimetrics.
\newblock In \emph{International Conference on Learning Representations}, 2022.

\bibitem[Wang et~al.(2022)Wang, Du, Torralba, Isola, Zhang, and
  Tian]{wang2022denoisedmdps}
Tongzhou Wang, Simon~S. Du, Antonio Torralba, Phillip Isola, Amy Zhang, and
  Yuandong Tian.
\newblock Denoised mdps: Learning world models better than the world itself.
\newblock In \emph{International Conference on Machine Learning}. PMLR, 2022.

\bibitem[Zhang et~al.(2020)Zhang, McAllister, Calandra, Gal, and
  Levine]{zhang2020learning}
Amy Zhang, Rowan McAllister, Roberto Calandra, Yarin Gal, and Sergey Levine.
\newblock Learning invariant representations for reinforcement learning without
  reconstruction.
\newblock \emph{arXiv preprint arXiv:2006.10742}, 2020.

\end{thebibliography}

\appendix

\section{Deriving IQE From PQE}\label{apd:iqe-pqe}

Here we will derive IQE via modifying the PQE-LH formula to scale linearly with latent (\ie, to have latent positive homogeneity).

Recall the PQE-LH formula from \citet{wang2022learning}: \begin{equation}
\eqname{PQE-LH}
    d_\mathsf{PQE\hbox{-}LH}(u, v; \alpha) = \sum_i \alpha_i \cdot (1 - \exp(-\sum_{j} (u_{ij} - v_{ij})^+)). \label{eq:pqe-lh}
\end{equation}

To make it scale linearly with latents, we must avoid the exponentiation transform on latent vector values, and instead use the latent vector to control a linear quantity. Therefore, we will reformulate the outer sum as an integral, and use latent vector to indicate where the summand (now integrand) has non-zero values. 

First, we reformulate \Cref{eq:pqe-lh} with an integration without weighting (by $\alpha$):\begin{equation}
\eqname{Integral PQE-LH}
    d_\mathsf{Integral\hbox{-}PQE\hbox{-}LH}(u, v) = \int_x (1 - \exp(-\sum_{j} (h_{j}(u; x) - h_{j}(v; x))^+)) \diff x.
\end{equation}

PQE-LH is derived by considering processes only activated on sets of the form $[x, \infty)$ (half-lines). Inspired by this choice, we consider $h_j(u; x) = \begin{cases} c & \text{ if } x > u_j \\ 0 & \text{ otherwise} \end{cases}$, for some $c > 0$.

Then \begin{equation}
    d_\mathsf{Integral\hbox{-}PQE\hbox{-}LH}(u, v; c) = \int_x (1 - \exp(- c \cdot \size{ \{j \colon x \in [u_j, \max(u_j, v_j)]\} }) \diff x.
\end{equation}

Take $c \rightarrow \infty$, we have \begin{align}
    d_\mathsf{Integral\hbox{-}PQE\hbox{-}LH}(u, v) 
    & = \int_x  \indic{\exists j, x \in  [u_j, \max(u_j, v_j)]} \diff x \\
    & = \size{\bigcup_j\ [u_j, \max(u_j, v_j)]},
\end{align}
which is exactly the IQE component.

Then, for expressivity, we combine several such components and obtain IQEs.

\section{Proofs}\label{apd:proofs}

\begin{proof}[\Cref{thm:iqe-maxmean-fin}]

\begin{itemize}
    \item \textbf{Proof for \text{IQE-maxmean}.}
    
    At $l=1$, \text{IQE-maxmean} formula can exactly recover the MRN asymmetrical component $d_\mathsf{asym}$. By Theorem~2 of \citet{liu2022metric}, $(f_1, d_\mathsf{asym})$ can exactly represent $d$ for some $f_1$. Therefore, the same results apply to $d_\mathsf{IQE\hbox{-}maxmean}$.

    \item \textbf{Proof for \text{IQE-sum}.}

    For $d_\mathsf{IQE\hbox{-}sum}$, we present a novel construction that allows it to represent any \qpart, and thus any convex combination of \qparts. Then, by Lemma~C.5 of \citet{wang2022learning}, some convex combination of \qparts admits a $\mathcal{O}(t \log^2 n)$ embedding.
    
    WLOG, consider any \qpart $\pi$ represented as an order embedding $g \colon \mathcal{X} \rightarrow [n]^m$. That is, \begin{equation}
        \pi(u, v) = \begin{cases}
        0 & \text{ if } g(u) \leq g(v) \text{ coordinate-wise} \\
        1 & \text{ otherwise.}
        \end{cases}
    \end{equation}
    Consider vectors $e_i \in \{0, 1\}^n$, where only the first $i$ dimensions are $0$'s, and the rest are $1$'s. These vector nicely connect the IQE component structure (union of intervals) with the order embedding structure (conjunction over coordinate-wise comparisons).
    
    For any latent $u, v$ and any $i \in [m]$, \begin{equation}
        \bigcup_{j=1}^n \big[ (e_{g_i(u)})_j, \max((e_{g_i(u)})_j, (e_{g_i(v)})_j) \big] = \begin{cases}
        \varnothing & \text{ if } g_i(u) \leq g_i(v) \\
        [0, 1] & \text{ otherwise.}
        \end{cases}
    \end{equation}

    Construct mapping \begin{equation}
        f(u) \trieq [e_{g_1(u)} :: e_{g_2(u)} :: \dots :: e_{g_m(u)} ] \in \{0, 1\}^{mn}, 
    \end{equation}
    where $::$ denotes concatenation.
    
    Then, for any latent $u, v$, \begin{equation}
        \bigcup_{j=1}^n \big[ f_i(u), \max(f_i(u), f_i(v)) \big] = \begin{cases}
        \varnothing & \text{ if } g(u) \leq g(v) \text{ coordinate-wise} \\
        [0, 1] & \text{ otherwise.}
        \end{cases}
    \end{equation}
    
    By using scaled $f$, each IQE component can thus represent arbitrary scaled \qpart. Thus \text{IQE-sum} can exactly represent any convex of \qparts using a polynomial-sized neural encoder.
\end{itemize}
\end{proof}

\begin{proof}[\Cref{thm:iqe-maxmean-inf}]
In proof of \Cref{thm:iqe-maxmean-fin}, a reduction from MRN asymmetrical part to \text{IQE-maxmean} is given. The same reduction can be applied here. Invoking Theorem~2 of \citet{liu2022metric} leads to the desired result.
\end{proof}

\begin{proof}[\Cref{thm:dnwn-ua}]
MRN approximation results (same as \Cref{thm:iqe-maxmean-fin,thm:iqe-maxmean-inf}) are proved showing that an asymmetric norm (\ie, semi-norm) universally approximate \qmets (Theorem~2 of \citealp{liu2022metric}). Deep Norm and Wide Norm can approximate any semi-norm (Theorem~2 of \citealp{pitis2020inductive}) and thus have the same properties.
\end{proof}

\section{Fixes for Deep Norm and MRN}

The original formulations of Deep Norm and MRN actually do not fully satisfy \qmet constraints. Here we highlight where they are wrong and explain our proposed fixes. In \Cref{sec:expr}, we compare with both the original and the fixed version.

\subsection{Deep Norm May Be Negative}\label{sec:dn-issue}

In the original work \citep{pitis2020inductive}, Deep Norm is formulated as a combinations over several components, each of which is the output of a $\mathrm{maxrelu}$ activation, where \begin{equation}
    \mathrm{maxrelu}(x, y; \alpha, \beta) \trieq [\max(x, y), \alpha \cdot \mathrm{ReLU}(x) + \beta \cdot \mathrm{ReLU}(y)].
\end{equation} 
However, the $\max$ component is not guaranteed to be non-negative, so the eventual output may be negative, and Deep Norm may not be a valid \qmet. 

To fix this issue, we simply replace the final activation to be simply $\mathrm{ReLU}$. As shown in \Cref{tab:berkstan,fig:q-learning-psucc-zoom}, this fix improves performance.

\subsection{MRN May Not Be A \Qmet}\label{sec:mrn-issue}

In the original work \citep{liu2022metric}, MRN is formulated as the sum of a symmetrical component and an asymmetrical component. While the asymmetrical component $\max_i (h(u)_i - h(v)_i)^+$ is a valid \qmet, the symmetrical component $\norm{\phi(u) - \phi(v)}_2^2$ is not a metric.

To fix this issue, we simply remove the square and use $\norm{\phi(u) - \phi(v)}_2$ as the symmetrical component. As shown in \Cref{tab:berkstan,fig:q-learning-psucc-zoom}, this fix improves performance.

\section{Experiment Details}\label{sec:expr-details}

Across all three tasks, our architecture choices and optimization settings generally follow the prior work \citep{wang2022learning}. For completeness, we report the full details below. We run each experiment setting for $5$ runs with different seeds, and present the aggregated results.

\subsection{Large-Scale Social Graph}

\paragraph{Architecture.} For all embedding methods (\ie, asymmetrical dot products and latent \qmets), we use a 128-2048-2048-2048-512 ReLU encoder with Batch Normalization \citep{ioffe2015batch} after each activation. The encoders take in $128$-dimensional inputs and output $512$-dimensional latent vectors. \Uncon networks use a similar 256-2048-2048-2048-512-1 ReLU network, mapping concatenated the $256$-dimensional input to a scalar output.

\paragraph{Optimization.} We use $80$ training epochs, batch size $1024$, and the Adam optimizer \citep{kingma2014adam}, with learning rate decaying from $10^{-4}$ to $0$ by the cosine schedule without restarting \citep{loshchilov2016sgdr}. The training objective is MSE on the $\gamma$-discounted distances, with $\gamma=0.9$. When applying the triangle inequality regularizer (for asymmetrical dot products and \uncon networks), $342 \approx 1024/3$ triplets are uniformly sampled at each iteration to compute the regularizer term.

\paragraph{Hyperparameters.} For the following baselines, we tune their hyperparameters:
\begin{itemize}[topsep=6pt, itemsep=-1pt]
    \item \textbf{IQE} and \textbf{PQE}: component size $l \in \{8, 16, 32, 64\}$ (and thus correspondingly number of components $k \in \{64, 32, 16, 8\}$).
    \item \textbf{Deep Norm (both the original version and the version with our fix)}: three layers with hidden size $\in \{128, 512\}$, where final number of output components equals the hidden size.
    \item \textbf{Wide Norm}: $32$ components each with size $\in \{32, 48, 128\}$.
    \item \textbf{MRN (both the original version and the version  with our fix)}: Both the symmetrical and the asymmetrical projection heads have two layers with hidden size $\in \{128, 512\}$, where the projector output dimension equals the hidden size.
    \item \textbf{Asymmetrical Dot Products and \Uncon Networks with $\Delta$-inequality regularizer}: regularizer weight $\in \{0.3, 1, 3\}$.
\end{itemize}

\subsection{Random Graphs}
\begin{figure}[t]
\centering
\floatconts
  {fig:graph-data-pdist}%
  {\caption{Structures of random graphs used in \Cref{sec:rand-graph} experiments.}}%
  {%
    \subfigure[\small Dense Graph.]{%
      \includegraphics[scale=0.63, trim=0 0 0 0]{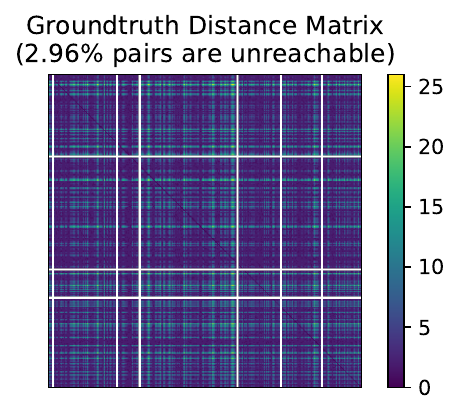}%
    }%
    \hfill%
    \subfigure[\small Sparse Graph.]{%
      \includegraphics[scale=0.63, trim=0 0 0 0]{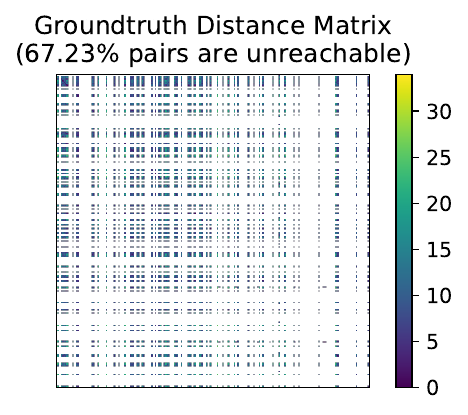}%
    }%
    \hfill%
    \subfigure[\small Sparse Graph with Block Structure.]{%
      \includegraphics[scale=0.63, trim=0 0 0 0]{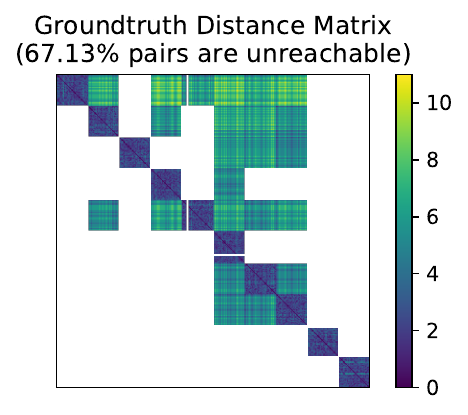}%
    }%
    \vspace{-20pt}%
  }
\end{figure}

\paragraph{Graph Structures and Datasets.} We use the same randomly generated $300$-node graphs and train-validation splits as \citet{wang2022learning}.  See \Cref{fig:graph-data-pdist} for a visualization on their structures. We test the models on $5$ different training set ratios, where training set fraction (of the total $90{,}000$ pairs) are evenly spaced on the logarithm scale from $0.01$ to $0.7$. We use the same $64$-dimensional node features from \citet{wang2022learning}.

\paragraph{Architecture.} For all embedding methods (\ie, asymmetrical dot products and latent \qmets), we use a 4-128-128-128-48 ReLU encoder, mapping to $48$-dimensional latent vectors. \Uncon networks use a similar 128-128-128-128-48-1ReLU network, mapping concatenated the $128$-dimensional input to a scalar output.

\paragraph{Optimization.} We use  $3000$ training epochs, batch size $2048$, and the Adam optimizer \citep{kingma2014adam}, with learning rate decaying  to $0$ by the cosine schedule without restarting \citep{loshchilov2016sgdr}. All models are optimized \wrt MSE on the $\gamma$-discounted distances, with $\gamma=0.9$. When running with the triangle inequality regularizer, $342 \approx 1024/3$ triplets are uniformly sampled at each iteration.

\paragraph{Hyperparameters.} For all methods, learning rates are tuned among $\{10^{-4}, 3 \times 10^{-4}, 10^{-3}, 3\times 10^{-3}, 10^{-2}\}$. We run with the following hyperparameters for each method:
\begin{itemize}[topsep=6pt, itemsep=-1pt]
    \item \textbf{IQE} and \textbf{PQE}: component size $l \in \{4, 6, 8, 12\}$ (and thus correspondingly number of components $k \in \{12, 8, 6, 4\}$).
    \item \textbf{Deep Norm (both the original version and the version with our fix)}: three layers with $48$ hidden size, where final number of output components equals the hidden size.
    \item \textbf{Wide Norm}: $12$ components each with size $11$, $22$ components each with size $6$, or $6$ components each with size $22$.
    \item \textbf{MRN (both the original version and the version  with our fix)}: Both the symmetrical and the asymmetrical projection heads have two layers with $58$ hidden size, where the projector output dimension equals the hidden size.
    \item \textbf{Asymmetrical Dot Products and \Uncon Networks with $\Delta$-inequality regularizer}: regularizer weight $\in \{0.3, 1, 3\}$.
\end{itemize}

\paragraph{\Cref{fig:random-graph} Details.}
For clarity, we didn't plot all settings for each family of method. Instead, we plot only the best member within the following families: IQE-sum, IQE-maxmean, PQE-LH, PQE-GG, Wide Norm, Deep Norm (original), Deep Norm (with fix), MRN (original), MRN (with fix), \uncon networks with each particular output parametrization (\ie, we tune triangle inequality regularizer weight for \uncon networks that output discounted distances, and only plot the best), asymmetrical dot products with each particular output parametrization (done in an identical way with \uncon networks). All metric embeddings are plotted.

\subsection{Offline Q-Learning}

\paragraph{Environment and Datasets.} Following \citet{wang2022learning}, we use the grid-world environment based on \texttt{gym-minigrid} \citep{gym_minigrid}, and use a training dataset of  trajectories, collected by an $\eps$-greedy planner with groundtruth \qmets, with a large $\eps = 0.6$, where each trajectory is capped at $200$ steps.

\paragraph{Algorithm.} Following \citet{wang2022learning}, we use a modified version of MBOLD \citep{tian2020model}. Please refer to \citet{wang2022learning} for details on the modifications.  To use encoder-based methods, we train them with goals as state-action pairs. In evaluation, for current state $s$, candidate action $a$ and a given goal state $g$, we use $\frac{1}{\size{\mathcal{A}}} d((s, a), (g,a')) - 1$ as the predicted distance from $(s,a)$ to $g$. For \uncon networks, we also test the original formulation where goals are simply states.

\paragraph{Architecture.} For all embedding methods (\ie, asymmetrical dot products and latent \qmets), we use a 18-2048-2048-2048-1024 ReLU encoder with Batch Normalization \citep{ioffe2015batch} after each activation. The encoders take in $18$-dimensional states and output four $256$-dimensional latent vectors, one for each actions. For \uncon networks, \begin{itemize}[topsep=6pt, itemsep=-1pt]
    \item With the new formulation using state-action pairs as goals, we use a similar 36-2048-2048-2048-256-16 network to map input state pairs to a value for each $\size{\mathcal{A}}\times\size{\mathcal{A}}$ action pair options;
    \item With the original formulation using states as goals, we use a similar 36-2048-2048-2048-256-4 network to map input state pairs to a value for each action.
\end{itemize}

\paragraph{Optimization.} We use $100$ training epochs, batch size $1024$, and the Adam optimizer \citep{kingma2014adam}, with learning rate decaying from $10^{-4}$ to $0$ by the cosine schedule without restarting \citep{loshchilov2016sgdr}. The training objective is the same as usual Q-learning: MSE on the $\gamma$-discounted distances, with $\gamma=0.95$. When applying the triangle inequality regularizer (for asymmetrical dot products and \uncon networks), $341 \approx 1024/3$ triplets are uniformly sampled at each iteration to compute the regularizer term.

\paragraph{Planning.} We  perform simple greedy 1-step planning without any lookahead. At each step, we query the learned Q-function for all action choices, and select the best action. In evaluation, we plan for $50$ goals, and cap trajectory length at $300$ steps.

\paragraph{Hyperparameters.} For each method, we run the following hyperparameter choices:
\begin{itemize}[topsep=6pt, itemsep=-1pt]
    \item \textbf{IQE}: component size $l = 8$ (and thus correspondingly number of components $k= 16$), selected based on effects of $(k,l)$ discussed in \Cref{sec:berkstan}.
    \item \textbf{PQE}: component size $l  = 4$ (and thus correspondingly number of components $k = 32$), following \citet{wang2022learning}.
    \item \textbf{Deep Norm (both the original version and the version with our fix)}: three layers with hidden size $\in \{64, 128\}$, where final number of output components equals the hidden size.
    \item \textbf{Wide Norm}: $32$ components each with size $\in \{32, 48, 128\}$.
    \item \textbf{MRN (both the original version and the version  with our fix)}: Both the symmetrical and the asymmetrical projection heads have two layers with hidden size $\in \{128, 512\}$, where the projector output dimension equals the hidden size.
    \item \textbf{Asymmetrical Dot Products and \Uncon Networks with $\Delta$-inequality regularizer}: regularizer weight $\in \{0.3, 1, 3\}$.
\end{itemize}
Each choice is plotted as a line in \Cref{fig:q-learning}. 

\end{document}